%% file: moead-tight-communication-bounds.tex
\documentclass{article}

\usepackage[preprint,nonatbib]{neurips_2026}

\usepackage[utf8]{inputenc} 
\usepackage[T1]{fontenc}    
\usepackage{hyperref}       
\usepackage{url}            
\usepackage{booktabs}       
\usepackage{amsfonts}       
\usepackage{nicefrac}       
\usepackage{microtype}      
\usepackage{xcolor}         

\usepackage{amsmath, amsthm, amsxtra, amssymb, amstext, mathtools}
\usepackage{thmtools} 
\usepackage{xspace}
\usepackage{enumitem}
\usepackage[algo2e,ruled,vlined,linesnumbered]{algorithm2e}\SetArgSty{upshape}
\usepackage{thm-restate}
\usepackage[nameinlink]{cleveref}

\usepackage[shortcuts]{extdash} 

\newcommand{\mylineref}[1]{\hyperref[#1]{line~\ref*{#1}}}

\makeatletter
\@ifundefined{newcounteralias}{}{%
  \renewcommand\thmt@autorefsetup{%
    \@xa\def\csname\thmt@envname autorefname\@xa\endcsname
      \@xa{\thmt@thmname}%
  }%
}
\makeatother

\newtheorem{theorem}{Theorem}

\newtheorem{lemma}[theorem]{Lemma}
\newtheorem{corollary}[theorem]{Corollary}

\newcommand{\oea}{\mbox{${(1 + 1)}$~EA}\xspace}

\newcommand{\ooea}{\oea}

\DeclareRobustCommand{\NSGAtwo}{NSGA\nobreakdash-II\xspace}
\DeclareRobustCommand{\NSGAthree}{NSGA\nobreakdash-III\xspace}
\DeclareRobustCommand{\SMS}{SMS\nobreakdash-EMOA\xspace}
\newcommand{\SPEA}{SPEA2\xspace}

\newcommand{\omm}{\textsc{OMM}\xspace}

\newcommand{\lotz}{\textsc{LOTZ}\xspace}

\newcommand{\R}{\ensuremath{\mathbb{R}}}

\newcommand{\N}{\ensuremath{\mathbb{N}}} 

\let\originalleft\left
\let\originalright\right
\renewcommand{\left}{\mathopen{}\mathclose\bgroup\originalleft}
\renewcommand{\right}{\aftergroup\egroup\originalright}

\definecolor{orange}{RGB}{255,127,0}

\DeclarePairedDelimiter\ceil{\lceil}{\rceil}
\newcommand{\abs}[1]{\left| #1\right|\xspace}

\title{Understanding the Co-evolutionary Solution Exchange of the MOEA/D: Already the Smallest Communication Radius Gives the Best-Possible Speed-Up
}
\author{
  Benjamin Doerr$^1$, Moaad Elmoutassim$^2$, Martin~S. Krejca$^1$, Guy Mushiata Kibulu$^2$\\
  $^1$ Laboratoire d'Informatique (LIX), CNRS, École Polytechnique,\\ Institut Polytechnique de Paris, Palaiseau, France\\
  $^2$ École Polytechnique, Institut Polytechnique de Paris, Palaiseau, France\\
  \texttt{firstname.all-lastnames@polytechnique.edu}
}

\begin{document}

\maketitle

\begin{abstract}
  Besides domination-based algorithms such as the \NSGAtwo, the decomposition-based MOEA/D is one of the most popular and successful optimization heuristic for multi-objective problems. It decomposes the problem into several, single-objective subproblems. By solving these in a co-evolutionary manner and archiving non-dominated solutions, it approximates or computes the Pareto front. Interestingly, so far no advantages from the co-evolutionary communication of subproblem solutions were proven. In this first mathematical runtime analysis addressing this research question, we investigate how the mutation-based MOEA/D solves the classic \omm benchmark when solutions are exchanged among subproblems in different distances~$r$. We show, for the first time, that the MOEA/D profits from an exchange of solutions, and that already very small communication radii $r$ are sufficient to maximally profit from this feature of the algorithm. More precisely, for all values $r \ge 1$ we prove a tight runtime guarantee of $\Theta(n^2)$, where $n$ is the problem size. This is asymptotically faster than the known bound of $O(n^2 \log n)$ guarantee for $r=0$, which we complement with a matching bound of $\Omega(n^2 \log n)$.
\end{abstract}

\section{Introduction}

Many real-world problems require the optimization of multiple competing objectives.
This naturally leads to various, mutually incomparable, optimal trade-offs, known as \emph{Pareto optima}, where improving one objective necessarily reduces the quality of another objective.
Thus, the overall optimization goal is to quickly find as many Pareto optima as possible.
To this end, multi-objective evolutionary algorithms (MOEAs) are widely and very successfully applied~\cite{ZhouQLZSZ11,CoelloLV07}.
Their population-based approach of maintaining a set of currently best-so-far solutions lends itself well to the multi-objective setting, and their general-purpose nature allows for great flexibility.

State-of-the-art MOEAs can be classified into two main approaches. Using Pareto dominance as main selection criterion and then a secondary criterion ensuring a balanced coverage of the Pareto front is the strategy
of the \NSGAtwo~\cite{DebPAM02}---the most popular MOEA, with over $60{,}000$ citations on Google Scholar.
Other popular MOEAs using this general strategy include the \NSGAthree~\cite{DebJ14}, \SMS~\cite{BeumeNE07}, and \SPEA~\cite{ZitzlerLT01}. All these algorithms are intensively used in practice. Since more recently, they are also subject to intensive theoretical research, see \cite{ZhengD23aij,WiethegerD23,BianZLQ23,RenBLQ24} and follow-up works.

An alternative and more recent strategy is to consider a decomposition of the objective function~$f$ into several less complex subproblems, and obtain a good set of solutions for $f$ by optimizing these.
The most popular MOEA making use of this approach is the \emph{decomposition-based MOEA} (MOEA/D)~\cite{ZhangL07}, which decomposes~$f$ into a user-specified number of single-objective subproblems.
The MOEA/D proceeds by optimizing these subproblems instead of~$f$ directly.
However, the best-so-far solutions of each subproblem are stored in a shared archive that stores all non-dominated solutions according to the original objective function~$f$, thus ensuring that high-quality solutions for~$f$ are obtained.
The subproblems themselves are optimized in a co-evolutionary fashion, where each subproblem exchanges its best-so-far solution with all subproblems of a user-specified \emph{communication radius}~$r$.
The intention of~$r$ is to benefit from potential positive correlations among the different subproblems.

The MOEA/D has been subject to rigorous mathematical investigation for considerably longer than dominance-based MOEAs~\cite{LiZZZ16}, however, the theoretical understanding of this algorithm seems less developed (see \Cref{sec:relatedWork} for a more detailed discussion). In particular, we have to this day no theoretical results on the necessity of or benefits from the co-evolutionary solution exchange in the MOEA/D.
All theoretical results so far conduct their analysis as if there was no exchange of solutions. Hence this central aspect of the MOEA/D, distinguishing it from a simple parallel optimization of the subproblems, is not at all understood via theoretical means by now.

\paragraph{Our Contribution.}
We prove tight bounds for how the expected runtime of the MOEA/D depends on the choice of its communication radius~$r$.
As our main result, we show that already a radius $r = 1$ leads to asymptotic speed-ups when compared to a trivial communication radius of $r = 0$.
In particular, as we discuss at the end of \Cref{sec:relatedWork}, the guarantees that we obtain with non-trivial communication are better than that of various state-of-the-art MOEAs based on non-dominated sorting, when using a comparable parametrization. Interestingly, we obtain the same runtime guarantees for all $r \ge 1$, that is, increasing $r$ further does not lead to further speed-ups when waiting for the entire Pareto front to be witnessed.
However, for achieving partial progress, we show that the runtime scales linearly in~$\frac{1}{\sqrt{r}}$, the details of which we discuss at the beginning of \Cref{sec:runtimeAnalysis}.

In more detail, we study how the mutation-only MOEA/D optimizes the bi-objective OneMinMax (\omm) benchmark of size~$n$, when using the Chebyshev decomposition into $n + 1$ subproblems with uniform weights.
Hence, each Pareto optimum is represented by the optimum of exactly one subproblem. We further assume that the reference point is initially set to the ideal value of $(n,n)$. This is a standard setting for runtime results for the MOEA/D.

In our main result (\Cref{thm:moeadOmmQuadraticRuntime}), we show that if the MOEA/D uses a communication radius of $r \geq 1$, it exhibits an expected runtime of $\Theta(n^2)$ function evaluations until its archive witnesses the entire Pareto front of \omm for the first time.
This result improves the previously best known bound of $O(n^2 \log n)$~\cite{LiZZZ16} by a factor of~$O(\log n)$, and our bound is asymptotically tight.
Our result follows from a careful analysis of the exchange of solutions among subproblems.
We show that an extremal subproblem propagates its solution quickly to sufficiently many other subproblems.
In turn, this increases the probability to pick a suitable subproblem for finding an improving solution.
The amount of how many subproblems can actually benefit from such an exchange is limited by the structure of \omm, which is why the expected runtime is the same for all communication radii greater or equal to~$1$.
However, the benefit of larger search radii is observable to some extent, before this capping occurs.

We prove \Cref{thm:moeadOmmQuadraticRuntime} for both a uniform and cyclical way of selecting which subproblem to choose next in an iteration.
This shows that the exact choice in which communication occurs is not relevant.
Moreover, we note that we prove the upper and lower quadratic bound separately (in \Cref{thm:moeadQuadraticUpperBound,thm:moeadQuadraticLowerBound}), for more general types of mutation, of which standard bitwise mutation---the most widely applied type of mutation---is a special case (\Cref{lem:bitwiseMutationBothTypes}).

We complement our quadratic bound in \Cref{cor:moeadOmmNoCommunicationSuperQuadraticRuntime} with a bound of $\Theta(n^2 \log n)$ for a trivial communication radius of $r = 0$.
This result follows directly by combining known results in the literature for similar processes.
It shows that the speed-up for $r \geq 1$ is, indeed, a result of the communication radius, since all other parameters remain the same as in the setting for $r \geq 1$.

\paragraph{Outline.}
We discuss related work in \Cref{sec:relatedWork}.
In \Cref{sec:preliminaries}, we introduce our mathematical notation and terminology, and we present the MOEA/D.
In \Cref{sec:runtimeAnalysis}, we discuss and prove our tight $\Theta(n^2)$ runtime bound for non-trivial communication radii.
We complement this result in \Cref{sec:trivialCommunicationRadius} with a $\Theta(n^2 \log n)$ runtime bound for trivial radii.
Last, we conclude our paper in \Cref{sec:conclusion}.

\section{Related Work}
\label{sec:relatedWork}

The MOEA/D has been studied mathematically from several angles so far.
However, none of these results studied the impact of the communication radius on the runtime.
In fact, none of the papers take the communication among subproblems into account when conducting a runtime analysis.

The first runtime analysis of the MOEA/D was conducted by Li et~al.~\cite{LiZZZ16}.
Besides other results, the authors prove an $O(n^2 \log n)$ bound for the expected runtime on the \omm benchmark\footnote{The authors erroneously refer to this benchmark as COCZ in their paper.} in the same setting as in this article.

Huang and Zhou~\cite{HuangZ20} study the MOEA/D with contiguous hypermutation, which is a non-standard mutation operator from artificial immune systems that performs larger changes more frequently than standard operators.
The authors study a general decomposition of the objective functions into~$N$ subproblems, where~$N$ is a user-specified value.
Important for the comparison to our work, Huang and Zhou~\cite{HuangZ20} prove that the MOEA/D with contiguous hypermutation optimizes the \omm benchmark in $O(N n^2 \log n)$ expected time.
Hence, even for a minimal value of $N = 2$, their upper bound only matches the upper bound of Li et~al.~\cite{LiZZZ16} for the classic MOEA/D.

Huang et~al.~\cite{HuangZCHLX21} consider the MOEA/D with standard bitwise mutation and with the slightly problem-specific one-point crossover, also in the setting of a general decomposition into~$N$ subproblems.
Relevant to our article, the authors prove an $O(N n \log n)$ bound for the \omm benchmark.
For the minimal choice of $N = 2$, this results in an $O(n \log n)$ bound, which is faster than the traditional $O(n^2 \log n)$ by Li et~al.~\cite{LiZZZ16} by a factor of~$O(n)$.
The reason for this speed-up is that the crossover operator allows to easily find all Pareto optima once the extremal solutions of the Pareto front have been found.
While impressive, this work highlights the benefit of crossover, and it does not focus on communication among the solutions.

Similar to the result by Huang et~al.~\cite{HuangZCHLX21}, Doerr, Krejca, and Weeks~\cite{DoerrKW24} also study the MOEA/D with a general decomposition into $N + 1$ subproblems, studying only the \omm benchmark.
The authors analyze the impact of mutation, studying both standard bitwise mutation and fast mutation, the latter of which performs changes following a power law with exponent $\beta > 1$.
They prove for the version with bitwise mutation an expected $O(n N \log n + n^{n/(2N)} N \log n)$ runtime\footnote{The authors complement a specific part of this runtime with a loose lower bound that is also exponential in $n / N$.} and for the version with fast mutation an $O(n N \log n + n^\beta \log n)$ expected runtime.
As with the previous results, this result highlights how different operators can lead to a far higher tolerance with respect to the number of problem decompositions, in particular making the setting $N = 1$ feasible.
Nonetheless, the communication radius is not discussed.

Huang et~al.~\cite{HuangZLL21} study the mutation-only MOEA/D for problems with $m \geq 2$ objectives.
They are interested in the first time that the algorithm finds any optimal solution, which is very different from the standard setting of witnessing the entire Pareto front at once.
The authors consider three different decompositions, into~$m$ subproblems each, studying the performance of the MOEA/D on the $m$\lotz and $m$COCZ benchmarks.
The main result of the paper is that the weighted-sum decomposition performs drastically outperforms the other two decomposition methods.
This paper explicitly sets the communication radius to the lowest-possible value.

Zhang et~al.~\cite{ZhangZZD26} also consider the mutation-only MOEA/D for a flexible number~$N$ of subproblems.
Similar to Huang et~al.~\cite{HuangZLL21}, the authors study the impact of the decomposition type, but they consider the standard runtime definition, which is the time until the archive witnesses the entire Pareto front.
Most important to our work, Zhang et~al.~\cite{ZhangZZD26} show that with a logarithmic number of subproblems, the MOEA/D with the weighted-sum decomposition has an expected runtime of $O(n \log^2 n)$ on the \omm benchmark, which is considerably faster than the typical $O(n^2 \log n)$ bound.
Again, this highlights the importance of the decomposition choice, but the communication radius is not studied. The latest work studying the MOEA/D on classic benchmarks~\cite{DoerrKW26} proves that a self-adjusting way to adapt the decomposition can lead to drastic runtime gains for unevenly scaled Pareto fronts.

In the only work on a classic optimization problem, Do et~al.~\cite{DoNNS23} study how the MOEA/D solves multi-objective minimum-weight-base problems.
The authors introduce a special decomposition for their problem and prove that it finds the entire Pareto front in an expected time that is fixed-parameter tractable.

Besides the MOEA/D, other MOEAs have been theoretically studied.
Most relevant to our work, Doerr and Qu~\cite{DoerrQ23LB} showed that the \NSGAtwo has a lower bound of $\Omega(N n \log n)$ on \omm, where $N \geq 4(n + 1)$ denotes the population size of the algorithm.
Thus, for an optimal choice of~$N$, the algorithm has an $\Omega(n^2 \log n)$ lower bound, which is tight~\cite{ZhengLD22}.
The same bound of $\Theta(n^2 \log n)$ also holds for the global SEMO, another MOEA that is commonly studied theoretically~\cite{BossekS24}.
We note that while for other algorithms based on non-dominated sorting, when relying solely on mutation, only a comparable $O(n^2 \log n)$ for best-possible parameter values is known~\cite{WiethegerD24arxiv,WiethegerD24,ZhengLD22,RenBLQ24,WiethegerD23,ZhengD24}, even for variants that are more robust in their parameter settings~\cite{DoerrIK25,DoerrKS26}, there is little reason to believe that these algorithms can optimize \omm faster than $\Omega(n^2 \log n)$ when only using mutation.

Outside the theory domain, some suggestions on the optimal size of communication radius exists, but due to their experimental nature, they are less clear than a typical theoretical result. In the work \cite{ZhangL07} proposing the MOEA/D, generally the MOEA/D was perceived as not very sensitive to the size of the neighborhood used for the exchange of solutions. However, the authors observed that for a knapsack instance, a large radius was detrimental, whereas for the ZDT1 benchmark, a large radius was preferable. A more detailed study on the influence of the communication radius~\cite{IshibuchiAN13} finds that it depends on the number of objectives, but that generally a small constant fraction of the population size such as $10$ to $20$ percent gives good results. We are not aware of more detailed results, in particular, we have found none that discuss whether the coevolutionary solution exchange is advantageous at all.

\section{Preliminaries}
\label{sec:preliminaries}

Let~$\N$ denote the set of natural numbers, including~$0$, and let~$\R$ be the set of reals.
For all $a, b \in \R$, let $[a .. b] = [a, b] \cap \N$.
For all $m \in \N_{\geq 1}$, $v \in \R^m$, and $i \in [1 .. m]$, we denote the value of~$v$ at component~$i$ by~$v_i$.

We investigate \emph{pseudo-Boolean} bi-objective problems, that is, for a given $n \in \N_{\geq 1}$, we consider the \emph{maximization} of \emph{objective functions} $f\colon \{0, 1\}^n \to \R^2$.
We call each $x \in \{0, 1\}^n$ an \emph{individual} and $f(x)$ its \emph{objective value}.
We refer to the value of~$x$ in the first and second objective, respectively, by~$f_1(x)$ and~$f_2(x)$.
Last, we let~$|x|_0$ and~$|x|_1$ denote, respectively, the number of~$0$s and~$1$s of~$x$.

When we mention objective functions, we assume that the dimension $n \in \N_{\geq 1}$ is implicitly given, and when we use big-O notation, we consider asymptotics in this~$n$.

\subsubsection*{Multi-Objective Terminology}

We compare objective values via the \emph{weak dominance} preorder~$\succeq$.
For all $u, v \in \R^2$, we say that~\emph{$v$ weakly dominates~$u$}, written $v \succeq u$ if and only if $v_1 \geq u_1$ and $v_2 \geq u_2$.
In all other cases, we say that~\emph{$u$ and~$v$ are incomparable}.

Given an objective function~$f$ of dimension $n \in \N_{\geq 1}$, we use the dominance terminology also when comparing individuals, referring to their objective values.
We call each $x \in \{0, 1\}^n$ \emph{Pareto-optimal} if and only if for all $y \in \{0, 1\}^n$, we have that $x \succeq y$ or that~$x$ and~$y$ are incomparable.
If~$x$ is Pareto-optimal, we call~$f(x)$ a \emph{Pareto optimum}.
The set of all Pareto optima of~$f$ is its \emph{Pareto front}.

\subsubsection*{The OneMinMax Benchmark}

The OneMinMax (\omm) benchmark \cite{GielL10} is the most widely studied objective function for multi-objective evolutionary algorithms, and it is typically the first function that is considered when conducting a new type of analysis, as we are doing.

As the name suggests, \omm aims at both minimizing and maximizing the number of~$1$s.
Formally, $\omm\colon \{0, 1\}^n \to [0 .. n]^2$ is defined as $x \mapsto (|x|_0, |x|_1)$.
Note that each individual is Pareto-optimal, and the Pareto front is $\{(k, n - k) \mid k \in [0 .. n]\}$, that is, it has $n + 1$ objective values.

\subsubsection*{The MOEA/D Algorithm}

The \emph{decomposition-based multi-objective evolutionary algorithm} (MOEA/D, \Cref{alg:moead}) \cite{ZhangL07} decomposes a given objective function~$f$ via a user-defined number $N \in \N_{\geq 1}$ into $N + 1$ single-objective optimization problems $(g_k)_{k \in [0 .. N]}$.
Each subproblem is subject to \emph{minimization}.
Moreover, for each subproblem $k \in [0 .. N]$, a single best-so-far individual~$x^{(k)}$ is maintained, each of which is initialized uniformly at random.
Last, the MOEA/D maintains an archive~$P$ of best-so-far individuals.

Iteratively, the MOEA/D selects in a user-defined manner a subproblem $k \in [0 .. N]$ and attempts to improve~$x^{(k)}$ by creating a new individual~$y$, called \emph{offspring}.
If~$g_k(y)$ is at least as good as~$g_k(x^{(k)})$, then~$x^{(k)}$ is replaced by~$y$.
Moreover,~$y$ is also communicated to all subproblems whose index is within a certain user-defined radius $r \in \N_{\geq 1}$ of~$k$, and their best-so-far individual is also potentially updated.
Finally, the archive is updated with respect to~$y$.

\textbf{Subproblem definition.}
We consider the \emph{Chebyshev} decomposition of problem~$f$, which is one of the most common MOEA/D decompositions.
Given a weight vector $\lambda \in [0, 1]^{N + 1}$ in increasing order as well as a reference point $z^* \in \R^2$, the Chebyshev decomposition aims at minimizing the weighted $L_{\infty}$-distance of a given objective value to the reference point, where each subproblem receives a different weight.
Formally, for each $k \in [0 .. N]$, subproblem~$k$ is defined for all $x \in \{0, 1\}^n$ as
\begin{align*}
  g_k(x, z) = \min\{\lambda_k \cdot \abs{f_1(x) - z^*_1}, (1 - \lambda_k) \cdot \abs{f_2(x) - z^*_2}\} .
\end{align*}

In our analysis of~\omm, we choose $z^* = (n, n)$ to reflect the best-possible value in each objective.
Moreover, we use a uniform weight vector, that is, for all $k \in [0 .. N]$, we choose $\lambda_k = \frac{k}{N}$.
We note that, together with the assumption $N = n$, we thus obtain for all $k \in [0 .. N]$ and all $x \in \{0, 1\}^n$ that
\begin{align}
  \label{eq:ommChebyshev}
  g_k(x)
  = \min\left\{\frac{k}{n} (n - |x|_0), \frac{n - k}{n} (n - |x|_1)\right\}
  =
  \begin{cases}
    \frac{k (n - |x|_0)}{n} & \text{if $|x|_0 < k$, and} \\
    \frac{|x|_0 (n - k)}{n} & \text{if $|x|_0 \geq k$.}
  \end{cases}
\end{align}
In particular, this implies for each $k \in [0 .. n]$ that all individuals with exactly~$k$ $0$s are minimizers for~$g_k$, as both terms in the minimum are equal.
Hence, each subproblem corresponds to exactly one Pareto optimum of \omm.

\textbf{Subproblem selection.}
We consider two types of selecting a subproblem: uniformly at random and cyclically, the former of which is self-explanatory.

When choosing a subproblem cyclically, the algorithm chooses consecutively a subproblem from $[0 .. N]$, starting in the first iteration with subproblem~$0$, and going after subproblem~$N$ back to~$0$.

\textbf{Bitwise mutation.}
Given an individual~$x$ as input (the \emph{parent}), bitwise mutation creates an offspring~$y$ by first copying~$x$ and then deciding independently for each position $i \in [1 .. n]$ of~$y$ whether to flip~$y_i$; each flip occurs with probability~$\frac{1}{n}$.

\begin{algorithm2e}[t]%
  \caption{
    The MOEA/D {\protect\cite{ZhangL07}} maximizing an objective function $f\colon \{0, 1\}^n \to \R^2$, given decomposition number $N \in \N_{\geq 1}$, subproblems $\{g_k\}_{k \in [0 .. N]}$, reference point $z \in \R^2$, and communication radius $r \in \N_{\geq 1}$.
  }
  \label{alg:moead}

  $P \gets \emptyset$\;
  \For{$k \in [0 .. N]$}{%
    choose $x^{(k)}$ uniformly at random from $\{0, 1\}^n$\;
    $P \gets P \cup \{x^{(k)}\}$\;
  }
  remove all elements from~$P$ dominated by other elements in~$P$\;

  \While{stopping criterion is not met}{
    choose a subproblem $k \in [0 .. N]$\;
    $y \gets$ create offspring from~$x^{(k)}$ via bitwise mutation\;\nllabel{line:mutation}
    \For{$i \in [k - r .. k + r] \cap [0 .. N]$}{%
      \lIf{$g_i(y, z) \leq g_i(x^{(i)}, z)$}{%
        $x^{(i)} \gets y$%
      }
    }
    remove all elements weakly dominated by~$y$ from~$P$ and add~$y$ to~$P$ if it is not weakly dominated by an element of $P$\;
  }
  \Return{$P$\;}
\end{algorithm2e}

\subsubsection*{Runtime on \omm}

We define the \emph{runtime} of the MOEA/D with decomposition number $N = n \in \N_{\geq 1}$ on \omm of problem dimension~$n$ to be the random number of \emph{function evaluations} (of~$\omm$) until the entire Pareto front of~$\omm$ is witnessed by the archive~$P$ of the MOEA/D, that is, $\{(k, n - k) \mid k \in [0 .. n]\} \subseteq \{\omm(x) \mid x \in P\}$.

We assume that each individual created by the MOEA/D is evaluated once upon its creation and then never again.
Hence, since the MOEA/D creates $n + 1$ individuals initially and then exactly one offspring in each iteration, its runtime is $n + 1$ plus the number of iterations until termination.

\section{Tight Runtime Guarantees for Non-Trivial Communication Radii}
\label{sec:runtimeAnalysis}

Our main result is \Cref{thm:moeadOmmQuadraticRuntime} below, which proves a tight expected runtime of $\Theta(n^2)$ function evaluations for the MOEA/D on \omm when using a communication radius of at least~$1$.

\begin{theorem}
  \label{thm:moeadOmmQuadraticRuntime}
  Let $n \in \N_{\geq 1}$.
  Consider the MOEA/D (\Cref{alg:moead}) with decomposition number~$n$, Chebyshev decomposition, reference point $(n, n)$, communication radius $r \in [1 .. n]$, and either uniform or cyclical subproblem selection, optimizing \omm.
  Then the expected runtime is $\Theta(n^2) = O\left(\frac{n^2}{\sqrt{r}} + n^2\right) \cap \Omega(n^2)$ function evaluations.
\end{theorem}

\Cref{thm:moeadOmmQuadraticRuntime} shows a speed-up of order~$\ln n$ compared to the same setting but with a trivial communication radius of~$0$.
The reason for this speed-up is that individuals with a currently best-so-far value in one objective are successively communicated to neighboring subproblems, for which these individuals are also best-possible at that time.
This accumulation of such extremal individuals improves the rate at which an individual is chosen that has a decent chance to expand the Pareto front.
In contrast, for communication radius~$0$, each subproblem needs to find such an extremal individual by itself.

More formally\footnote{We assume uniform subproblem selection in this discussion, noting that cyclical selection, when viewed over~$n$ consecutive iterations, behaves similarly to uniform selection.}, we consider subproblem~$0$ and assume that its number of~$0$s is $k \in [1 .. \frac{n}{2}]$.
Except for at most once during the optimization, if subproblem~$0$ has a value of at most~$k$, so do subproblems~$1$ to~$\min\{k, r\} \eqqcolon M$, where the cap at~$k$ is due to the selection preference stated by \Cref{lem:preference}.
We call an individual with fewer than~$k$ $0$s an \emph{improving} individual, for this discussion.
When choosing subproblems~$0$ to~$M$ for offspring creation and not creating an improving individual but instead a copy, that is, an individual with exactly~$k$ $0$s, this individual is the new best-so-far individual for at least subproblem $M + 1$.
Analogously, subproblems~$1$ to $M + 1$, if chosen, communicate an individual with exactly~$k$ $0$s with constant probability to at least subproblem $M + 2$, if it exists, and so on.
Since each subproblem is chosen with probability~$\frac{1}{n}$, the communication spreads, in expectation, to $\ell \in [1 .. k]$ subproblems in $\ell \frac{n}{r}$ iterations.
Since creating an improving offspring from an individual with exactly~$k$ $0$s has a probability of order~$\frac{k}{n}$, if $\ell \in [1 .. k]$ subproblems in the population have an individual with exactly~$k$ $0$s, the probability to create an improving individual is of order~$k \ell / n^2$ and thus takes~$n^2 / (k \ell)$ iterations in expectation.
Hence, if $\ell \approx \sqrt{n r / k}$, then these two expected times coincide, resulting in finding an improving solution in order $n^{3/2} / \sqrt{r k}$, provided $k \geq (n r)^{1/3}$.
For smaller~$k$, it takes order $n^2 / k^2$ iterations in expectation, as there are not enough subproblems available for which a solution with exactly~$k$ $0$s is the best-so-far solution while no improvement is found.
Due to this limit, larger communication radii do not provide any further asymptotic speed-up when aiming to witness the entire Pareto front.
In other words, while there are enough solutions to reach, the upper runtime bound is $O(n^2 / \sqrt{r})$, and the algorithm profits from larger search radii.
However, at some point in time, this effect is capped, and the remaining runtime is $O(n^2)$, which is independent of~$r$, save for $r \geq 1$.

The lower bound follows a different line of thought.
Assume that subproblem~$0$ does not have an optimal individual yet, but subproblems~$1$ and~$2$ do, which we prove to occur very likely.
Then all remaining subproblems have a best-so-far solution of at least two~$0$s and, in particular, discard an individual with exactly one~$1$.
Consequently, only subproblems~$0$ and~$1$ have, each, a probability of order~$\frac{1}{n}$ to create the optimum for subproblem~$0$.
This results in at least order~$n^2$ iterations in expectation in order to solve subproblem~$0$.
All other subproblems have at most a chance of order~$\frac{1}{n^2}$, also resulting in a quadratic lower bound.

In the following, we prove the upper and lower bound of \Cref{thm:moeadOmmQuadraticRuntime} separately, following the outline above.
For each of these bounds, rather than only considering bitwise mutation, we consider a general class of mutation operators, which is always suited to our proof method.
Bitwise mutation is just one operator that lies in the intersection of these larger classes.

\subsection{A Quadratic Upper Bound for Non-Trivial Communication Radii}
\label{sec:upperBound}

For the upper bound of \Cref{thm:moeadOmmQuadraticRuntime}, we rely on communicating copies of individuals with currently extremal single-objective values among neighboring subproblems, as explained in \Cref{sec:runtimeAnalysis}.
Our analysis generalizes to a certain type of mutation operators that we call \emph{type-A operators}.

We say that a mutation operator $\pi\colon \{0, 1\}^n \to \{0, 1\}^n$ is a \emph{type-A operator} if and only if there exist constants $C_A, \alpha \in \R_{> 0}$ such that for all $x \in \{0, 1\}^n$, we have
\begin{itemize}
  \item $\Pr(|\pi(x)|_0 = |x|_0-1) \ge C_A\frac{|x|_0}{n}$,
  \item $\Pr(|\pi(x)|_0 = |x|_0+1) \ge C_A\frac{n-|x|_0}{n}$, and
  \item $\Pr(|\pi(x)|_0 = |x|_0) \ge \alpha$.
\end{itemize}

Our main result is the following.

\begin{theorem}
  \label{thm:moeadQuadraticUpperBound}
  Let $n \in \N_{\geq 1}$.
  Consider the MOEA/D (\Cref{alg:moead}) with decomposition number~$n$, Chebyshev decomposition, reference point $(n, n)$, communication radius $r \in [1 .. n]$, and either uniform or cyclical subproblem selection, optimizing \omm.
  Assume that the MOEA/D makes use of a type-A mutation operator in \mylineref{line:mutation}.
  Then the expected runtime is $O\left(\frac{n^2}{\sqrt{r}} + n^2\right)$ function evaluations.
\end{theorem}

In order to prove \Cref{thm:moeadQuadraticUpperBound}, we define more formally for each subproblem $k \in [0 .. n]$ when an individual replaces the best-so-far individual of subproblem~$k$.

Let $k \in [0 .. n]$.
For all $x, x' \in \{0, 1\}^n$, we say that~$x$ is \emph{at least as attractive as~$x'$ for subproblem~$k$}---denoted by $x \preceq_k x'$---if and only if $g_k(x) \le g_k(x')$.
Furthermore, if and only if $g_k(x) < g_k(x')$, we say that~$x$ is \emph{strictly more attractive} than~$x'$ for subproblem~$k$.

The following lemma shows for all subproblems $k \in [0 .. n]$ how the attractiveness of an individual is related to its distance of its numbers of~$0$s to~$k$.

\begin{lemma}\label{lem:preference}
  Let $k \in [0 .. n]$ and $x, x' \in \{0, 1\}^n$.
  If $|x|_0,|x|_0'\ge k$, then $x\preceq_k x'$ if and only if $|x|_0 \le |x'|_0$.
  Likewise, if $|x|_0,|x|_0'\le k$, then $x\preceq_k x'$ if and only if $|x|_0\ge |x|_0'$.

  Moreover, if $k \le \frac{\sqrt{n}}{2}$, $|x|_0 \le k$, and $|x'|_0 > k$, then $x \prec_k x'$.
\end{lemma}

\begin{proof}
  First, note that \cref{eq:ommChebyshev} is only depends on the number of~$0$s of its argument.
  Hence, for all $y \in \{0, 1\}^n$, let $g'_k(|y|_0) \coloneqq g_k(y)$.
  The first two statements follow then directly from the fact that~$g'_k$ is decreasing on $[0 .. k]$ and increasing on $[k .. n]$.

  We now consider the case $k \le \frac{\sqrt{n}}{2}$, $|x|_0 \le k$, and $|x'|_0 > k$, relying again on the definition of~$g'$ as introduced above.
  Let $j = |x|_0$ and $j' = |x'|_0$, and let $\ell \in \N$ and $r \in \N_{\geq 1}$ such that $j=k-\ell$ and $j'=k+r$.
  Using the piecewise form of $g'_k$, we obtain
  \begin{align*}
    g'_k(j)\le g'_k(j')
    \iff \frac{k(n-(k-\ell))}{n}\le \frac{(k+r)(n-k)}{n}
    \iff \ell \le r\,\frac{n-k}{k} .
  \end{align*}

  Since $k\le \frac{\sqrt{n}}{2}$, we have $\frac{n-k}{k} \ge \frac{2n-\sqrt{n}}{\sqrt{n}} > \frac{\sqrt{n}}{2} \ge k \ge \ell$.
  Hence, since $r\ge 1$, we have $\ell < r\,\frac{n-k}{k}$ and thus $g'_k(k-\ell)< g'_k(k+r)$, concluding the proof.
\end{proof}

\begin{restatable}{lemma}{upperBoundTimeToExtremes}
  \label{lem:upperBound:timeToExtremes}
  Consider the same setting as in \Cref{thm:moeadQuadraticUpperBound}.
  Then the expected runtime for solving subproblems~$0$ and~$n$ is $O\left(\frac{n^2}{\sqrt{r}} + n^2\right)$ function evaluations.
\end{restatable}

The proof of \Cref{lem:upperBound:timeToExtremes} follows the same outline as mentioned at the beginning of \Cref{sec:runtimeAnalysis}. Details are deferred to the appendix.
Letting $k \in [1 .. n]$ denote number of~$0$s of the best-so-far individual of subproblem~$0$, we bound the expected time until (1)~subproblems~$1$ to $\ell_k = \min \{\ceil{\sqrt{n r / k}}, k\}$ also have a best-so-far solution with exactly~$k$ $0$s or (2)~the minimum number of~$0$s among these subproblems decreases.
By \Cref{lem:preference}, for each subproblem with index in $[0 .. \ell_k]$, a solution with exactly~$k$ $0$s is at least as attractive as a larger value of the subproblem while we did not find an improvement.
Since copies are created with a constant probability, by the definition of a type-A operator, and since we inform either all relevant subproblems immediately or at least~$r$ of them every $O(\frac{n}{r})$ iterations on average, it takes $O(\frac{n}{r} \ell_k)$ in expectation in order to propagate a solution with exactly~$k$ $0$s to~$\ell_k$ subproblems (or to find an improvement earlier).

Given~$\ell_k$ subproblems as above, the expected time for finding an improvement is $O(\frac{n^2}{k \ell_k})$, as we have a probability of at least~$\frac{\ell_k}{n}$ of choosing a suitable subproblem, each of which has a probability of at least order~$\frac{k}{n}$ to flip exactly one of the~$k$ $0$s.
Substituting~$\ell_k$ and summing over all values for~$k$ concludes the proof.

We show that once the extremal subproblems are solved, the remaining subproblems are also solved in at most quadratic time.

\begin{restatable}{lemma}{upperBoundfromExtremesToCenter}
  \label{lem:upperBound:fromExtremesToCenter}
  Consider the same setting as in \Cref{thm:moeadQuadraticUpperBound}.
  Then once the subproblems~$0$ and~$n$ are solved, the expected runtime for witnessing the entire Pareto front of \omm is $O(n^2)$ function evaluations.
\end{restatable}

\begin{proof}
  We show that, starting from subproblem~$0$, we solve all subproblems in $[0 .. \ceil{\frac{n}{2}}]$ in expected time $O(n^2)$.
  For the other subproblems, an analogous argument applies by swapping all~$0$s and~$1$s in the argument, as explained also initially in the proof of \Cref{lem:upperBound:timeToExtremes}.
  Hence, the overall expected time is at most twice the time that we estimate for the lower half.

  We show for all $k \in [0 .. \ceil{\frac{n}{2}}]$ that once subproblem~$k$ is solved, we solve subproblem $k + 1$ in expected time $O(n)$, leading to an overall expected time of $O(n^2)$ when summing over all values of~$k$.
  To this end, we use the same terminology as in the proof of \Cref{lem:upperBound:timeToExtremes}, calling the number of~$0$s of the best-so-far individual of a subproblem its \emph{value}.

  Consider an iteration such that $k \in [0 .. \ceil{\frac{n}{2}} - 1]$ is maximal such that for all $i \in [0 .. k]$, we have that subproblem~$i$ is solved.
  When selecting subproblem~$k$, then by the definition of a type-A operator, mutation produces value $k + 1$ with probability at least $C \frac{n - k}{n} \geq \frac{C}{2}$.

  For uniform selection, subproblem~$k$ is selected, in expectation, in $O(n)$ iterations.
  For cyclical selection, it is selected for sure in~$n$ iterations.
  Hence, for either subproblem selection, we solve subproblem $k + 1$ in expected time $O(n)$, concluding the proof.
\end{proof}

\Cref{thm:moeadQuadraticUpperBound} follows by first applying \Cref{lem:upperBound:timeToExtremes} and then \Cref{lem:upperBound:fromExtremesToCenter}.

\subsection{A Quadratic Lower Bound for Non-Trivial Communication Radii}
\label{sec:lowerBound}

Similar to \Cref{sec:upperBound}, we introduce a general mutation operator, called \emph{type-B operator}, before we state our quadratic lower bound.
In contrast to a type-A operator, which requires a decent probability for copies and for exploring the $1$-bit neighborhood of an individual, a type-B operator guarantees that the all-$1$s individual is not created too likely when its parent has at least one~$1$.

We say that a mutation operator $\pi\colon \{0, 1\}^n \to \{0, 1\}^n$ is a \emph{type-B operator} if and only if for all $x \in \{0, 1\}^n$, we have
\begin{itemize}
  \item if $|x|_0 = 1$, then $ \Pr(|\pi(x)|_0 \neq 0) \geq 1 - O\left(\frac{1}{n}\right) $,
  \item if $|x|_0 = 1$, then there is a constant $C_B \in (0, 1)$ such that for all $i \in [2 .. n]$, we have $\Pr(|\pi(x)|_0 \in [2, i] \mid |\pi(x)|_0 \neq 1) \geq 1 - C_B\left(\frac{1}{i^2} + \frac{1}{n}\right) > 0$, and
  \item if $|x|_0 \geq 2$, then $ \Pr(|\pi(x)|_0 \geq 1) \geq 1 - O\left(\frac{1}{n^2}\right) $.
\end{itemize}

Our main result in this section is the following lower bound.

\begin{restatable}{theorem}{moeadQuadraticLowerBound}
  \label{thm:moeadQuadraticLowerBound}
  Let $n \in \N_{\geq 1}$.
  Consider the MOEA/D (\Cref{alg:moead}) with decomposition number~$n$, Chebyshev decomposition, reference point $(n, n)$, communication radius $r \in [1 .. n]$, and either uniform or cyclical subproblem selection, optimizing \omm.
  Assume that the MOEA/D makes use of a type-B mutation operator in \mylineref{line:mutation}.
  Then the expected runtime is $\Omega(n^2)$ function evaluations.
\end{restatable}

We prove \Cref{thm:moeadQuadraticLowerBound} by showing that with constant probability, during the first~$n^2$ iterations, the all-$1$s individual is not created.
To this end, we show that with high probability, none of the $n + 1$ initial individuals is the all-$1$s individual.
Afterward, if a subproblem $i \in [2 .. n]$ is selected for the first time and its best-so-far individual has exactly one~$1$, its offspring has, with constant probability, at least two~$1$s, which are at least as attractive and thus retained, by \Cref{lem:preference}.
Moreover, we show that with constant probability, subproblems~$0$ and~$1$ are chosen at most~$3n$ times during the first~$n^2$ iterations, and each time that they are chosen, their offspring has at least one~$0$.
Last, we show that with constant probability, whenever during the first~$n^2$ iterations a subproblem is selected whose currently best-so-far solution has at least two~$0$s, its offspring has at least one~$0$.
This shows the claim.
Details are deferred to the appendix.

\subsection{Proof of the Tight Quadratic Runtime Guarantee}

\Cref{thm:moeadOmmQuadraticRuntime} is a consequence of \Cref{thm:moeadQuadraticUpperBound,thm:moeadQuadraticLowerBound} when knowing that bitwise mutation is both a type-A and a type-B operator.
We show this with the following statement.

\begin{lemma}
  \label{lem:bitwiseMutationBothTypes}
  For $n \in \N_{\geq 16}$, bitwise mutation is a type-A and a type-B operator.
\end{lemma}

\begin{proof}
  Let $x \in \{0, 1\}^n$ be the parent, and let $\pi\colon \{0, 1\}^n \to \{0, 1\}^n$ denote bitwise mutation.
  We show the properties of the two types separately.

  \paragraph*{Type-A operator.}
  In order for the offspring to have one fewer~$0$ than~$x$, exactly one of the~$0$s of~$x$ needs to flip, and all other bits must not flip.
  Consequently, we obtain $\Pr(|\pi(x)|_0 = |x|_0 - 1) = \frac{|x|_0}{n} (1 - \frac{1}{n})^{n - 1} \geq \frac{1}{e} \frac{|x|_0}{n}$.
  For increasing the number of~$0$s by one, an analogous reasoning applies, requiring to flip exactly one of the $n - |x|_0$ $1$s of~$x$, satisfying both bounds with $C_A = \frac{1}{e} \in \R_{> 0}$.

  Last, in order to create a copy, no bit of~$x$ must flip.
  Thus, we have $\Pr(|\pi(x)|_0 = |x|_0) = (1 - \frac{1}{n})^n \geq \frac{1}{e}(1 - \frac{1}{n}) \geq \frac{1}{2e}$, also satisfying the last property.

  \paragraph*{Type-B operator.}
  If $|x|_0 = 1$, then in order for the offspring not to be the all-$1$s individual, it is sufficient that the single~$0$ in~$x$ does not flip.
  Thus, we obtain $\Pr(|\pi(x)|_0 \neq 0) \geq 1 - \frac{1}{n}$, as required.

  For the next property, let $i \in [2 .. n]$, and assume that $|x|_0 = 1$.
  We estimate the probability that the offspring has a number of~$0$s in $[2 .. i]$, conditional on the offspring not being a copy.
  To this end, let $p_c \coloneqq (1 - \frac{1}{n})^n$ denote the probability of creating a copy.
  Let~$Z$ denote the unconditional event that the offspring is the all-$1$s individual, and let~$B$ denote the unconditional event that the offspring has more than~$i$ $0$s.
  Then we have $\Pr(|\pi(x)|_0 \in [2, i] \mid |\pi(x)|_0 \neq 1) = \bigl(1 - p_c - \Pr(Z) - \Pr(B)\bigr) / (1 - p_c) = 1 - (\Pr(Z) + \Pr(B)) / (1 - p_c)$.
  We continue with bounding the probabilities of~$Z$ and~$B$.

  For~$Z$, the single~$0$ in~$x$ must flip and none of the other bits.
  Hence, $\Pr(Z) = \frac{1}{n} (1 - \frac{1}{n})^{n - 1} = \frac{1}{n (1 - 1 / n)} (1 - \frac{1}{n})^n \leq \frac{1}{e} \frac{1}{n - 1}$.
  For~$B$, it is necessary that at least~$i$~$1$s in~$x$ flip.
  Hence, since $i \geq 2$, we have $\Pr(B) \leq \frac{1}{n^i} \binom{n - 1}{i} \leq \frac{(n - 1)^i}{n^i} \frac{1}{i!} \leq \frac{1}{i (i - 1)} \leq \frac{2}{i^2}$.

  By $1 - p_c \geq 1 - \frac{1}{e}$ and combining the bounds above, we obtain $\Pr(|\pi(x)|_0 \in [2, i] \mid |\pi(x)|_0 \neq 1) \geq 1 - \frac{2}{1 - 1/e} (\frac{1}{i^2} + \frac{1}{n})$, which is strictly positive for $i \geq 2$ and $n \geq 16$, proving this property.

  Last, assume $|x|_0 \geq 2$.
  In order for the offspring to be the all-$1$s individual, at least all~$0$s in~$x$ must flip.
  Thus, we have $\Pr(|\pi(x)|_0 \geq 1) \geq 1 - \frac{|x|_0}{n^{|x|_0}} \geq 1 - \frac{2}{n^2}$, concluding the proof.
\end{proof}

We note that while \Cref{lem:bitwiseMutationBothTypes} requires $n \in \N_{\geq 16}$, the expected runtime of the MOEA/D on \omm for each value $n \in [1 .. 15]$ is constant, thus also proving our claim for these cases.

\section{A Super-Quadratic Lower Bound for Trivial Communication Radius}
\label{sec:trivialCommunicationRadius}

We complement \Cref{thm:moeadOmmQuadraticRuntime} by showing that its runtime performance is worse by a factor of order~$\ln n$ if no communication exchange among the subproblems is used.

\begin{corollary}
  \label{cor:moeadOmmNoCommunicationSuperQuadraticRuntime}
  Let $n \in \N_{\geq 1}$.
  Consider the MOEA/D (\Cref{alg:moead}) with decomposition number~$n$, Chebyshev decomposition, reference point $(n, n)$, communication radius $r = 0$, and either uniform or cyclical subproblem selection, optimizing \omm.
  Then the expected runtime is $\Theta(n^2 \log n)$ function evaluations.
\end{corollary}

\Cref{cor:moeadOmmNoCommunicationSuperQuadraticRuntime} follows directly from well-known results.
The upper bound is given by Li et~al.~\cite[Proposition~$4$]{LiZZZ16}, following from an $n$-times application of the coupon collector problem.
The lower bound follows from noting that solving subproblem~$0$ takes $\Omega(n \log n)$ function evaluations if only considering function evaluations spent on this subproblem.
This is a special case of the fact that the $(1 + 1)$~evolutionary algorithm (EA) takes $\Omega(n \log n)$ function evaluations on each objective function with a unique optimum \cite[Theorem~$1$]{DrosteJW98ecj}.
Since the MOEA/D with communication radius~$0$ runs an independent \ooea for each subproblem, this result translated directly to subproblem~$0$.
The overall lower bound follows by noticing that, on average, we require~$n+1$ function evaluations in order to perform~$1$ function evaluation for subproblem~$0$.

\section{Conclusion}
\label{sec:conclusion}

We derived tight runtime guarantees for the MOEA/D (\Cref{alg:moead}) when optimizing the classic \omm benchmark of size~$n$.
\Cref{thm:moeadOmmQuadraticRuntime} shows that for any communication radius $r \geq 1$, the expected runtime is $\Theta(n^2)$ function evaluations, which improves the previously best bound of $O(n^2 \log n)$~\cite{LiZZZ16}.
Our improvement is a result of studying the process more carefully, showing that currently best-so-far individuals are propagated to sufficiently many subproblems quickly.
This results in a higher chance of selecting a useful subproblem for finding new Pareto-optima.
However, as our lower bound shows, the exact size of~$r$, once at least~$1$, is irrelevant, as the number of subproblems that share the same sub-optimal individuals is limited by the problem structure of \omm.

We contrast \Cref{thm:moeadOmmQuadraticRuntime} with \Cref{cor:moeadOmmNoCommunicationSuperQuadraticRuntime}, showing that if $r = 0$, the expected runtime is $\Theta(n^2 \log n)$.
Hence, the speed-up for $r \geq 1$ is, indeed, due to the exchange of individuals among subproblems.

As usual in the mathematical runtime analysis of randomized search heuristics \cite{NeumannW10,AugerD11,Jansen13,ZhouYQ19,DoerrN20}, we were able to analyze rigorously only one particular but commonly studied situation, and for arbitrary value of the parameter $r$ and for a broader class of mutation operators. From the working principles of our proofs, we are optimistic that the co-evolutionary solution exchange is advantageous also in other settings. This is clear when crossover is used with some constant probability $p_c < 1$, simply because our proof is monotonic in the sense that the creation of other individuals at most slows down the proven progress by the time spent on these other individuals. Similarly, when using the PBI decomposition with weight $\theta > 1$, we obtain the same set of subproblems, and hence the same process. The weighted-sum decomposition leads to a particular behavior on the \omm benchmark that is most likely not representative for a typical use of the MOEA/D~\cite{ZhangZZD26}. For the number of subproblems $N$ different from the size $n+1$ of the Pareto front, it is somewhat clear that $N > n+1$ simply implies that some points of the Pareto front are targeted by more than one subproblem. This should give a similar behavior as shown in our work except for that all numbers of individuals are scaled up appropriately. When $N < n+1$, for any value of the communication radius, there is the risk that points of the Pareto front are skipped and then are hard to find. Apart from this, the co-evolutionary solution exchange should work as in our setting, that is, the outermost solution is communicated to a large number of subproblems, which then in parallel search for more outermost solutions.

What is less clear from our work, and an interesting next step is to study, is for which other benchmarks such a speed-up exists and whether the runtime differs for different non-trivial communication radii. A natural next benchmark is \lotz, for which, different to \omm, only very few individuals are on the Pareto front.
Deriving in such a setting a central lemma like \Cref{lem:preference}, which succinctly summarizes when a subproblem updates its best-so-far individual, seems to be a logical next step.

Another natural question is to study an additional positive effect of crossover.
In this setting, since~$r$ also determines which subproblems are chosen for the crossover, the exact value of~$r$ may play a more important role. In particular, it could be that a small value of $r$ leads to a more suitable parent selection than a uniform parent selection as often used in domination-based MOEAs.

Last, we currently lack mathematical guarantees for the importance of the communication radius for problems with at least three objectives.
So far, there is only a single work studying such a setting for the MOEA/D, and it does not consider Pareto front coverage~\cite{HuangZLL21}.

\subsection*{Acknowledgments}

Martin S. Krejca is supported by the French National Agency for Research (ANR) via the JCJC grant titled MultiEDA (ANR-25-CE23-2418-01).
This research is also supported by the FMJH Program Gaspard Monge for optimization and operations research and their interactions with data science.

\bibliographystyle{alpha}
\bibliography{ich_master.bib,alles_ea_master.bib,rest.bib}

\newpage
\input{checklist.tex}

\newpage
\appendix
\section{Appendix for ``Understanding the Co-evolutionary Solution Exchange of the MOEA/D: Already the Smallest Communication Radius Gives the Best-Possible Speed-Up''}

This appendix contains the proofs of \Cref{lem:upperBound:timeToExtremes} and \Cref{thm:moeadQuadraticLowerBound} and is meant to be read at the reviewer's discretion only.
For the sake of convenience, we repeat the lemmas before we state their proof.

\upperBoundTimeToExtremes*

\begin{proof}
  Note that the discussion for subproblem~$n$ is the same as for subproblem~$0$ when swapping all~$0$s and~$1$s in the discussion, due to the symmetry of \cref{eq:ommChebyshev} and of type-A operators in terms of these numbers.
  Thus, we consider in the following only the case for subproblem~$0$.
  The case for subproblem~$n$ follows immediately and increases the total expected number of function evaluations for solving both subproblems by at most a factor of~$2$.

  Let $k \in [1 .. n]$.
  In the following, assuming that the number of~$0$s of the currently best-so-far individual for subproblem~$0$ is~$k$, we derive an upper bound on the expected time until we find a solution with at most $k - 1$ $0$s.
  This value is then communicated to subproblems~$1$ to~$r$, which are updated if they have a worse value.
  The expected number of fitness evaluations until solving subproblem~$0$ is then, by linearity of expectation, at most the sum of all of these expected waiting times.

  For the sake of brevity, we refer to the number of~$0$s of the best-so-far solution for any subproblem as its \emph{value}, and we consider the single-objective function $g'_0\colon [0 .. n] \to \R$ such that, for all $x \in \{0, 1\}^n$, we have $g'_k(|x|_0) = g_k(x)$.

  We first estimate the time until sufficiently many individuals have the same value as subproblem~$k$.
  Afterward, we estimate the time it takes in order to improve the value of subproblem~$0$.
  Last, we conclude.

  \paragraph{Communicating copies.}
  Let $\ell_k = \min \{\ceil{\sqrt{n r / k}}, k\}$ be the number of subproblems to which we communicate value~$k$, and let $M = \min \{r, \ell_k\}$ be the index of the largest subproblem such that, except for at most one during the run, all subproblems~$0$ to~$M$ have at most value~$k$, which we detail further below.
  Moreover, let $i \in [M .. \ell_k]$ be the largest index such that for all $i' \in [0 .. i]$, subproblem~$i'$ has value at most~$k$.
  We estimate the expected time until $i \geq \ell_k$ by estimating the time it takes for~$i$ to increase by one and then adding up these~$\ell_k$ waiting times.
  Note that, except if subproblem~$0$ is the initial solution, subproblems~$0$ to~$M$ all have value at most~$k$, because when subproblem~$0$ is updated, the following happens:
  Consider subproblem $j \in [1 .. M]$, noting that $j \leq k$.
  If subproblem~$j$ had a value greater than~$k$ before, then by \Cref{lem:preference}, when updating subproblem~$0$, the solution to subproblem~$0$ is at least as attractive to subproblem~$i$, and so subproblem~$i$ also has a value of~$k$.
  Otherwise, if subproblem~$j$ already has a value less than~$k$, then its value after the update is at most that of subproblem~$0$, which has value~$k$.
  We discuss the case that not all subproblems~$0$ to~$M$ have value at most~$k$ at the end of the proof, as this case occurs at most once during the run of the algorithm.

  Note that if $r \geq \ell_k$, then immediately subproblems~$0$ to~$\ell_k$ have value at most~$k$.
  Hence, the complexity in this case is bounded from above by any positive value.
  Thus, we assume in the following $r < \ell_k$.
  In this case, when any subproblem~$i$ to $i - M = i - r$ is selected, then by the definition of a type-A operator, with probability at least $\alpha>0$, this subproblem mutates and then propagates to at least subproblem $i + 1$ an individual of value $j \in [0 .. k]$ by creating a copy of its current best-so-far individual.
  Since~$i$ to $i - r$ is selected on average every $O(\frac{n}{r})$ iterations, this communication occurs after an expected $O(\frac{n}{r \alpha}) = O(\frac{n}{r})$ number of iterations.
  As discussed above, each such update leads to a value of at most~$k$ for subproblem $i + 1$.

  Since we need to reduce~$i$ at most $\ell_k - M$ times, index~$i$ is at least~$\ell_k$ after an expected time of $O\bigl(\frac{n}{r}(\ell_k - M)\bigr) = O(\frac{n}{r} \ell_k)$.

  Once at least~$\ell_k$ subproblems have a value of at most~$k$, we wait until at least one of these subproblems produces, via mutation, an individual with value at most $k-1$.
  To this end, we make a case distinction with respect to the subproblem selection mechanism.
  The communication radius is irrelevant for this step, since we only consider~$\ell_k$ subproblems.

  \paragraph{Uniform selection for reaching $k-1$.}
  The probability to pick a subproblem with value at most~$k$ and that mutation then produces an individual with a strictly smaller amount of~$0$s is by definition of a type-A operator at least $C_A \cdot \frac{k}{n} \cdot \frac{\ell_k}{n}$.
  Consequently, the expected time until such an improvement occurs is $O(\frac{n^2}{k\,\ell_k})$.

  \paragraph{Cyclical selection for reaching $k-1$.}
  The probability that no subproblem with value at most~$k$ mutates to value at most $k-1$ after all subproblems are selected once is, by the definition of a type-A operator, at most $\left(1 - C_A\frac{k}{n}\right)^{\ell_k}$.
  Thus, with probability at least $p_k \coloneqq 1 - (1 - C_A\frac{k}{n})^{\ell_k}$,
  at least one subproblem generates an individual of value at most $k-1$.
  The expected waiting time for this event is $O(n/p_k)$ iterations, because~$p_k$ considers~$n$ consecutive iterations.

  We bound $\frac{n}{p_k}$, noting that $\frac{k \ell_k}{n} \leq 1$.
  For $x \in (0, 1)$, using for all $t \in \R_{\geq 0}$ that $e^{-t}\le \frac{1}{1+t}$, we obtain for $t=xC_A$ that $
    \frac{1}{1-e^{-xC_A}}
    \le \frac{1}{1-\frac{1}{1+xC_A}}
    = \frac{1+xC_A}{xC_A}
    \le \frac{1+C_A}{xC_A}$.
  Thus,
  \begin{align*}
    \frac{n}{p_k}
    \le n \cdot (1+C_A)\frac{n}{C_Ak \ell_k}
    = O\left(\frac{n^2}{k \ell_k}\right) .
  \end{align*}

  \paragraph*{Reaching $k - 1$.}
  For either subproblem selection, we get in expected time $O(\frac{n^2}{k\ell_k})$ a solution with value at most $k-1$ for a subproblem in $[0 .. \ell_k]$.
  Then, using an analogous argument to the one for creating copies, we propagate this value to subproblem~$0$ in time $O(\frac{n}{M} \ell_k)$.
  Hence, in expected time $O\bigl(\frac{n}{M} \ell_k + \frac{n^2}{k \ell_k}\bigr)$, subproblem~$0$ has a value of at most $k - 1$.

  \paragraph*{Concluding.}
  Except for the initial case, which we discuss last, the overall runtime is bounded by $O\left(\frac{n}{r}\sum_{k=1}^{n} \ell_k+\sum_{k=1}^{n}\frac{n^2}{k\ell_k}\right)$.
  We estimate this quantity by making a case distinction with respect to~$k$.
  If $k \geq \ceil{(n r)^{1/3}}$, then $\ell_k = \ceil{\sqrt{n r / k}}$.
  The runtime contribution in this case amounts to $O\left(\frac{n}{r} \sum_{k = \ceil{(n r)^{1/3}}}^n \ceil{\sqrt{n r / k}} + n^2 \sum_{k = \ceil{n^{1/3}}}^n \frac{1}{k \ceil{\sqrt{n r / k}}}\right) = O\left(\frac{n^{3/2}}{\sqrt{r}} \sum_{k = \ceil{(n r)^{1/3}}}^n \frac{1}{\sqrt{k}}\right) = O\bigl(\frac{n^2}{\sqrt{r}}\bigr)$.

  If $k < \ceil{(n r)^{1/3}}$, then $\ell_k = k$, and we obtain a total runtime contribution of $O\left(\frac{n}{r} \sum_{k = 1}^{\ceil{(n r)^{1/3}}} k + n^2 \sum_{k = 1}^{\ceil{(n r)^{1/3}}} \frac{1}{k^2}\right) = O(n^2)$.

  For the final case that subproblem~$0$ contains the initial solution, we wait at most until a copy of this solution is made, as this communicates it immediately to subproblems~$1$ to~$M$, which then have a value of at most~$k$, as discussed above.
  By the definition of a type-A operator, this happens after an expected $O(\frac{n}{\alpha})$ iterations, which is negligible for the overall runtime contribution.
  This concludes the proof.
\end{proof}

\moeadQuadraticLowerBound*

\begin{proof}
  As in the proofs of \Cref{sec:upperBound}, we use the terminology that the \emph{value} of a subproblem is the number of~$0$s of its best-so-far individual.
  We prove that within $\Omega(n^2)$ iterations, if all subproblems have an initial value of at least~$1$, which occurs with overwhelming probability, then the all-$1$s individual is not produced with an at least constant probability.
  Hence, the expected runtime is at least quadratic.

  As sketched at the beginning of \Cref{sec:runtimeAnalysis}, we show that it is likely that most subproblems have a value of at least~$2$ before mutation even creates the all-$1$s individual, due to the properties of the type-B operator.
  In this case, at most the subproblems~$0$ and~$1$ have value~$1$, which is too few in order to create the all-$1$s individual with a probability that is sub-constantly close to~$1$.
  More formally, we aim at proving that the probability of a set of favorable events $Q = A \cap B \cap C \cap D \cap E$, defined in the following, satisfies $\Pr(Q) = \Theta(1)$.
  We define:
  \begin{itemize}
    \item \textbf{Event A.} All initial values are at least~$1$.
    \item \textbf{Event B.} For all $i \in [2 .. n]$, the first time that subproblem~$i$ is selected and has value~$1$, its offspring has a number of~$0$s in $[2 .. i]$.
    \item \textbf{Event C.} In the first~$n^2$ iterations, subproblem~$0$ or~$1$ is selected at most~$3n$ times.
    \item \textbf{Event D.} In the first~$n^2$ iterations, if subproblem~$0$ or~$1$ has value~$1$, its offspring has at least one~$0$.
    \item \textbf{Event E.} In the first~$n^2$ iterations, every time a selected subproblem has a value greater than or equal to~$2$, its offspring has at least one~$0$.
  \end{itemize}

  Assuming that all of these events have positive probability, which we show in the following, we obtain $\Pr(Q) =\Pr(A)\cdot \Pr(C \mid A) \cdot \Pr(D \mid A, C) \cdot \Pr(B\mid A,C,D) $$ $$ \cdot \Pr(E \mid A,B,C,D)$.
  We analyze each of these terms separately.

  During initialization, for each subproblem $k \in [0 .. n]$, the probability that subproblem~$k$ has the all-$1$s individual is $\frac{1}{2^n}$, independently of any other random choices.
  Hence, $\Pr(A) = \left(1 - \frac{1 }{2^n}\right)^{n+1} = \Theta(1)$.

  For uniform selection, let $N_{01} \in [0 .. n^2]$ be the random number of times that subproblem~$0$ or~$1$ are selected during the first~$n^2$ iterations.
  Hence,~$N_{01}$ follows a binomial law with~$n^2$ trials and a success probability of $2/(n+1)$, and it is independent of~$A$.
  Hence, by a Chernoff bound, it follows that $\Pr(N_{01} > 3n \mid A) \leq \exp(-\Omega(n))$.
  Thus, $\Pr(C\mid A) \geq 1 - \exp(-\Omega(n)) = \Theta(1)$.

  In cyclical selection, each element is selected every~$n$ iterations, so subproblem~$0$ or~$1$ are selected at most~$2n$ times in~$n^2$ iterations.
  Thus, $\Pr(C\mid A) = 1 = \Theta(1)$, too.

  Conditional on~$C$, subproblem~$0$ or~$1$ are selected at most~$3n$ times during the first~$n^2$ iterations.
  Each time they have a value of at least~$1$, by the definition of a type-B operator, the probability of the offspring being the all-$1$s individual is $O(1/n)$.
  Hence, $\Pr(D \mid A, C) \geq \left(1 - O\left(\frac{1}{n}\right)\right)^{3n} = \Theta(1)$.

  Let~$S$ denote the random variable representing the set of indices from $[2 .. n]$ such that for all $i \in S$, subproblem~$i$ triggers the event~$B$ during execution, that is, the first time that subproblem~$i$ is selected, its offspring has a number of~$0$s in $[2 .. i]$.
  Let~$i \in S$, and let~$p_i$ be the probability that the mutation that triggers event~$B$ of~$i$ produces an offspring with value in $[2 .. i]$.
  By the definition of a type-B operator, we have $p_i \geq 1 - C_B (1/i^2 + 1/n)$.
  Let $c^* = C_B (1/4 + 1/n)$.
  Noting that for all $\gamma, x \in (0, 1)$, the inequality $1 - x \geq e^{-\gamma x}$ is equivalent to $\gamma \geq \frac{1}{x} \ln(\frac{1}{1 - x})$, and noting that the function $x \mapsto \frac{1}{x} \ln(\frac{1}{1 - x})$ is monotonically increasing over $[0, 1]$, we see that for all constants $\gamma \in (0, 1)$ and for all $x \in [0 .. \gamma]$, we have $1 - x \geq  e^{-\frac{1}{\gamma} \ln(\frac{1}{1 - \gamma}) x}$, where $\frac{1}{\gamma} \ln(\frac{1}{1 - \gamma})$ is also a constant.
  Thus, there is a constant $c' \in (0, 1)$ such that for all $i \in S$, we have $1 - p_i \geq e^{-c' p_i}$.
  Using this inequality and the fact that mutations are independent of each other, we obtain
  \begin{align*}
     & \Pr(B\mid A,C,D)
    = E[\Pr(B \mid S,A,C,D)]
    = E\left[\prod\nolimits_{i \in S} p_i \mid A,C,D\right]                                     \\
     & \quad \geq \prod_{i=2}^{n} \left(1 - C_B \left(\frac{1}{i^2} + \frac{1}{n}\right)\right)
    \geq \exp\left(-c' \cdot C_B \sum_{i = 2}^{n} \left(\frac{1}{i^2} + \frac{1}{n}\right)\right)
    = \Theta(1) .
  \end{align*}

  For the last term in our probability bound for~$Q$, each time a subproblem with a value of at least~$2$ is selected, by the definition of a type-B operator, the offspring has at least one~$0$ with probability at least $1 - O(1/n^2)$.
  Hence, $\Pr(E \mid A,B,C,D) \geq \left(1 - O\left(\frac{1}{n^2}\right)\right)^{n^2} = \Theta(1)$.

  Altogether, we obtain $\Pr(Q) = \Theta(1)$, as desired.
  Last, note that any run of the MOEA/D conditional on~$Q$ lasts at least~$n^2$ iterations.
  By~$A$, each subproblem has initially a value of at least~$1$.
  By~$D$, if subproblem~$0$ or~$1$ is selected during the first~$n^2$ iterations, their offspring is never the all-$1$s individual.
  By~$E$, for each subproblem $i \in [2 .. n]$, if their value is at least~$2$ during the first~$n^2$ iterations, their offspring has at least one~$0$.
  And by~$B$, if subproblem~$i$ with value~$1$ creates an offspring, then its number of~$0$s is in $[2 .. i]$.
  By \Cref{lem:preference}, such an offspring is at least as attractive as the current best-so-far individual of subproblem~$i$.
  Hence, the offspring is the next best-so-far individual.
  This implies that subproblem~$i$ cannot have value~$1$ anymore.
  Consequently,~$E$ concludes the proof.
\end{proof}

\end{document}

%% file: checklist.tex
\section*{NeurIPS Paper Checklist}

\begin{enumerate}

      \item {\bf Claims}
      \item[] Question: Do the main claims made in the abstract and introduction accurately reflect the paper's contributions and scope?
      \item[] Answer: \answerYes{}
      \item[] Justification: Our contribution, which is clearly highlighted in the introduction, explicitly references our main results and mentions their scope.
            Moreover, we refer in the introduction explicitly to \Cref{sec:relatedWork}, where we detail how our paper relates to existing results, thus providing a good overview of the scope.

      \item {\bf Limitations}
      \item[] Question: Does the paper discuss the limitations of the work performed by the authors?
      \item[] Answer: \answerYes{}
      \item[] Justification: We discuss the limitations of our results in the introduction and in the conclusion.
            Moreover, the details in \Cref{sec:relatedWork} highlight differences to existing results, thus informing about limitations.

      \item {\bf Theory assumptions and proofs}
      \item[] Question: For each theoretical result, does the paper provide the full set of assumptions and a complete (and correct) proof?
      \item[] Answer: \answerYes{}
      \item[] Justification: We provide a proof or a detailed proof sketch for each mathematical claim in the paper, with the full proofs of \Cref{lem:upperBound:timeToExtremes} and \Cref{thm:moeadQuadraticLowerBound} given in the appendix.
            \Cref{thm:moeadOmmQuadraticRuntime} and \Cref{cor:moeadOmmNoCommunicationSuperQuadraticRuntime} follow directly from other results, which we state in sufficient detail, respectively at the end of \Cref{sec:runtimeAnalysis} (just before \Cref{sec:upperBound}) and of \Cref{sec:trivialCommunicationRadius}.

      \item {\bf Experimental result reproducibility}
      \item[] Question: Does the paper fully disclose all the information needed to reproduce the main experimental results of the paper to the extent that it affects the main claims and/or conclusions of the paper (regardless of whether the code and data are provided or not)?
      \item[] Answer: \answerNA{}
      \item[] Justification: This paper does not contain any experiments.

      \item {\bf Open access to data and code}
      \item[] Question: Does the paper provide open access to the data and code, with sufficient instructions to faithfully reproduce the main experimental results, as described in supplemental material?
      \item[] Answer: \answerNA{}
      \item[] Justification: This paper does not contain any experiments.

      \item {\bf Experimental setting/details}
      \item[] Question: Does the paper specify all the training and test details (e.g., data splits, hyperparameters, how they were chosen, type of optimizer) necessary to understand the results?
      \item[] Answer: \answerNA{}
      \item[] Justification: This paper does not contain any experiments.

      \item {\bf Experiment statistical significance}
      \item[] Question: Does the paper report error bars suitably and correctly defined or other appropriate information about the statistical significance of the experiments?
      \item[] Answer: \answerNA{}
      \item[] Justification: This paper does not contain any experiments.

      \item {\bf Experiments compute resources}
      \item[] Question: For each experiment, does the paper provide sufficient information on the computer resources (type of compute workers, memory, time of execution) needed to reproduce the experiments?
      \item[] Answer: \answerNA{}
      \item[] Justification: This paper does not contain any experiments.

      \item {\bf Code of ethics}
      \item[] Question: Does the research conducted in the paper conform, in every respect, with the NeurIPS Code of Ethics \url{https://neurips.cc/public/EthicsGuidelines}?
      \item[] Answer: \answerYes{}
      \item[] Justification: We carefully reviewed the NeurIPS Code of Ethics and concluded that none of our results or assumptions violate the code.

      \item {\bf Broader impacts}
      \item[] Question: Does the paper discuss both potential positive societal impacts and negative societal impacts of the work performed?
      \item[] Answer: \answerNA{}
      \item[] Justification: Our results provide theoretical guarantees for an existing algorithm.
            Hence, the algorithm itself, and thus its potential impact, is as potent as before, only our understanding of the algorithm improved.

      \item {\bf Safeguards}
      \item[] Question: Does the paper describe safeguards that have been put in place for responsible release of data or models that have a high risk for misuse (e.g., pre-trained language models, image generators, or scraped datasets)?
      \item[] Answer: \answerNA{}
      \item[] Justification: This paper does not deal with data or models.

      \item {\bf Licenses for existing assets}
      \item[] Question: Are the creators or original owners of assets (e.g., code, data, models), used in the paper, properly credited and are the license and terms of use explicitly mentioned and properly respected?
      \item[] Answer: \answerNA{}
      \item[] Justification: This paper does not use existing assets.

      \item {\bf New assets}
      \item[] Question: Are new assets introduced in the paper well documented and is the documentation provided alongside the assets?
      \item[] Answer: \answerNA{}
      \item[] Justification: This paper does release new assets.

      \item {\bf Crowdsourcing and research with human subjects}
      \item[] Question: For crowdsourcing experiments and research with human subjects, does the paper include the full text of instructions given to participants and screenshots, if applicable, as well as details about compensation (if any)?
      \item[] Answer: \answerNA{}
      \item[] Justification: This paper neither involves crowdsourcing nor research with human subjects.

      \item {\bf Institutional review board (IRB) approvals or equivalent for research with human subjects}
      \item[] Question: Does the paper describe potential risks incurred by study participants, whether such risks were disclosed to the subjects, and whether Institutional Review Board (IRB) approvals (or an equivalent approval/review based on the requirements of your country or institution) were obtained?
      \item[] Answer: \answerNA{}
      \item[] Justification: This paper neither involves crowdsourcing nor research with human subjects.

      \item {\bf Declaration of LLM usage}
      \item[] Question: Does the paper describe the usage of LLMs if it is an important, original, or non-standard component of the core methods in this research? Note that if the LLM is used only for writing, editing, or formatting purposes and does \emph{not} impact the core methodology, scientific rigor, or originality of the research, declaration is not required.
      \item[] Answer: \answerNA{}
      \item[] Justification: The research conducted in this paper does not involve LLMs.

\end{enumerate}